\documentclass[10pt]{article} 

\usepackage{natbib}
\usepackage{amsmath, amssymb, amsthm}
\usepackage{graphicx}
\usepackage{fullpage}
\usepackage{hyperref}
\usepackage{enumitem}
\usepackage{bm} 
\usepackage{authblk}

\usepackage{algorithm}
\usepackage{algorithmic}
\usepackage{booktabs}
\usepackage{xcolor}
\usepackage{microtype}

\definecolor{darkblue}{rgb}{0.0, 0.0, 0.5}  

\hypersetup{
    colorlinks=true,
    linkcolor=darkblue,
    anchorcolor=darkblue,
    citecolor=darkblue,
    filecolor=darkblue,
    menucolor=darkblue,
    urlcolor=darkblue
}

\usepackage[nameinlink,noabbrev]{cleveref}

\usepackage{caption}

\newtheorem{theorem}{Theorem}
\newtheorem{lemma}{Lemma}
\newtheorem{proposition}{Proposition}
\newtheorem{corollary}{Corollary}
\theoremstyle{definition}
\newtheorem{definition}{Definition}

\newtheorem{remark}{Remark}

\title{Exact Convex Reformulations of Linear Neural Networks \\ 
via Completely Positive Lifting}
\author[]{Karthik Prakhya}
\author[]{Alp Yurtsever}
\affil[]{Umeå University, Sweden \\ \vspace{0.25em} \texttt{karthik.prakhya@umu.se} \\ \texttt{alp.yurtsever@umu.se}}
\date{}

\begin{document}

\maketitle

\begin{abstract}
We show that the training problem of a deep linear neural network under the squared loss admits an exact convex reformulation in a lifted space over a generalized completely positive cone. The reformulation has the same optimal value as the original nonconvex problem and is linear in the lifted variables, with all nonconvexity encoded in the cone constraint. Its ambient lifted dimension depends only on the input and output dimensions, independent of the network depth and the number of data points, and the bottleneck width enters only through scalar constraints. The construction proceeds by reducing the multilayer parameterization to a bilinear factorization, lifting it to a rank-constrained semidefinite program, expressing the rank constraint via a complementarity condition, and applying a completely positive lifting. The resulting formulation gives a conic representation of the nonconvexity induced by linear factorization and connects linear neural network training with copositive programming.  
\end{abstract}

\section{Introduction}
\label{sec:introduction}
Consider the training problem of an $N$-layer linear neural network:
\begin{equation} 
    \min_{W_1, \dots, W_N}~
    \frac{1}{2n} \bigl\| W_N W_{N-1} \cdots W_1 X - Y \bigr\|_F^2,
    \label{eqn:deep-linear-net}
\end{equation}
where $X \in \mathbb{R}^{d_0 \times n}$ denotes the input data, $Y \in \mathbb{R}^{d_N \times n}$ the targets, and $W_k \in \mathbb{R}^{d_k \times d_{k-1}}$ the layer weight matrices. 
The network output is given by $\hat{Y} = W_N W_{N-1} \cdots W_1 X$. 
We denote the bottleneck width by \(r:=\min_{0\le k\le N} d_k\), including the input and output dimensions.

Although the input-output map is linear, the multilayer parameterization induces a nonconvex optimization problem due to the factorized structure of the weights.
Since compositions of linear mappings remain linear, depth does not increase expressivity, and problem~\eqref{eqn:deep-linear-net} is equivalent, at the level of represented functions (and hence also in terms of global optima) to a low-rank matrix factorization with bilinear structure~\citep{saxe2013exact,kawaguchi2016deep}. 

The main contribution of this paper is an exact convex reformulation of the linear neural network training problem under squared loss, obtained by specializing the lifting framework of \citet{bai2016conic}.
Starting from the standard reduced-rank representation, we reformulate the problem as rank-constrained semidefinite and complementarity formulations that fit this framework, leading to a lifted conic program over a generalized completely positive cone. 
The resulting formulation is exact at the level of optimal values and isolates the nonconvexity induced by factorization in a single cone constraint.
Its lifted variable dimension grows polynomially with the input and output dimensions and is independent of the network depth and the number of training samples. 
The data enter through the objective coefficients, while the bottleneck width appears only through scalar constraints.
Since reduced-rank regression already provides a direct solution method for the training problem studied here, the value of our reformulation is primarily structural. 
Specifically, it gives an exact conic representation of the nonconvexity induced by linear factorization.

\subsection{Related Work}

Linear neural network training under squared loss is closely related to reduced-rank regression, since the product of the network weights represents a linear map whose rank is bounded by the bottleneck width. 
The resulting unconstrained problem can be solved by classical reduced-rank regression methods based on singular value decompositions and the Eckart-Young theorem \citep{izenman1975reduced,reinsel1998multivariate}. 
We use this reduced-rank characterization as the starting point for deriving a convex lifted formulation through the generalized completely positive cone.

Convex formulations of neural networks have been studied through infinite-dimensional function-space representations \citep{bengio2005convex,bach2017convex}, as well as through convex descriptions of overparameterized feature distributions \citep{fang2022convex}. 
This perspective connects neural networks with reproducing kernel Banach spaces \citep{spek2025duality}, data-fitting problems over infinite-dimensional spaces \citep{rosset2007ell1, parhi2021banach, shenouda2024variation}, and function approximation theory \citep{barron1993universal, mhaskar2004tractability, meng2022new, liu2025integral}.

Recent work has shown that certain neural network training problems admit exact finite-dimensional convex formulations.
Convex reformulations for the nonconvex training of two-layer ReLU networks have been derived using convex semi-infinite duality and enumeration of ReLU sign patterns \citep{sahiner2020vector}.
Related techniques have also led to convex formulations for two- and three-layer convolutional neural networks with ReLU activations \citep{ergen2020implicit}.
For two-layer ReLU networks with weight decay, the training problem can be reformulated as a convex optimization problem with structured sparsity regularization \citep{pilanci2020neural}.
Extensions to multiple three-layer ReLU regularized subnetworks have shown that such models can be equivalently cast as convex optimization problems in higher-dimensional spaces \citep{ergen2021global}.
Further results characterize optimal hidden-layer weights for two-layer and certain deep ReLU networks as extreme points of suitable convex sets \citep{ergen2021convex,ergen2021revealing}.
Subsequent work has studied the geometry of these formulations and the associated convex landscapes \citep{ergen2023convexlandscape}.
These ideas are also applied to attention and related token-mixing modules in vision transformers, yielding finite-dimensional convex formulations for self-attention, MLP-Mixer, and Fourier Neural Operator mechanisms \citep{sahiner2022unraveling}.

Another major approach constructs convex relaxations through lifting techniques.
These methods represent parameters through higher-order moment matrices or Gram matrices, leading to semidefinite or completely positive programs.
Semidefinite liftings have been explored for training neural networks with polynomial activations \citep{bartan2021neural}.
This line of work was later extended to quantized neural networks with polynomial activations \citep{bartan2021training}, and to certain deep neural networks with polynomial and ReLU activations \citep{bartan2023convex}.
More recent work has explored convex formulations based on lifted representations for wide and shallow ReLU networks \citep{prakhya2024convex}.

Compared with the approaches above, our construction avoids activation-pattern indexing, whose size depends on the data geometry, for example through $\operatorname{rank}(X)$.
It gives an exact reformulation, and unlike lifted wide-network formulations, it applies to linear networks of arbitrary width, with the width appearing only through scalar constraints while the lifted cone dimension is determined by the input and output dimensions.

Our reformulation is obtained by specializing the lifting framework of \citet{bai2016conic} to the rank-constrained formulation arising from linear neural networks. 
This framework applies broadly to conic quadratically constrained quadratic programs and quadratic programs with complementarity constraints, including rank-constrained semidefinite programs. 
In this paper, we carry out this specialization for the training problem in~\eqref{eqn:deep-linear-net} and make the resulting conic representation explicit.
To the best of our knowledge, this gives the first exact generalized completely positive reformulation of the linear neural network training problem.

\subsection{Notation} 

We denote by $\mathbb{S}^p$ the space of $p \times p$ real symmetric matrices, and by $\mathbb{S}_+^p$ its positive semidefinite cone. 
For symmetric matrices $A,B \in \mathbb{S}^p$, we write $A \succeq B$
(or $A \preceq B$) if $A-B$ is positive semidefinite. 
We write $\operatorname{vec}(A)$ for the vector obtained by stacking the columns of $A$, and $\operatorname{mat}(\cdot)$ for its inverse, i.e., for $x \in \mathbb{R}^{d^2}$, $\operatorname{mat}(x) \in \mathbb{R}^{d \times d}$ satisfies $\operatorname{vec}(\operatorname{mat}(x)) = x$.
Throughout, $\langle\cdot,\cdot\rangle$ denotes the canonical inner product on the ambient space, in particular the Euclidean inner product for vectors and the Frobenius inner product for matrices. 
For inner product spaces $\mathbb{V}_1$ and $\mathbb{V}_2$, we write $\mathbb{V}_1 \oplus \mathbb{V}_2$ for their direct sum, equipped with the natural inner product. 
We use $\otimes$ to denote both the tensor and Kronecker product, depending on context. 

We work with lifted variables of the form $z \otimes z$ with $z \in \mathbb{V}$, which lie in the symmetric subspace of $\mathbb{V} \otimes \mathbb{V}$, i.e., the subspace invariant under exchanging the two tensor factors $(u \otimes v) \mapsto (v \otimes u)$, denoted by $\mathrm{Sym}^2(\mathbb{V})$. 
In particular, for $\mathbb{V} = \mathbb{R}^p$, this subspace can be identified with $\mathbb{S}^p$, and the tensor $z \otimes z$ corresponds to the rank-one matrix $zz^\top$ under this identification. 

\section{Exact Convex Reformulation}

Our main result establishes an exact convex reformulation for the training problem of a deep linear neural network in \eqref{eqn:deep-linear-net} based on generalized completely positive cones. 
We first recall the necessary definitions. 

\begin{definition}\label{def:completely-positive-cone}
The \textit{completely-positive cone} is defined as
$
    \mathcal{CP}^p := \operatorname{conv}\{zz^\top : z \in \mathbb{R}^p_+\}.
$
Its dual cone, the \textit{copositive cone}, is defined as 
$
    \mathcal{COP}^p = \{A \in \mathbb{S}^p : z^\top A z \ge 0,\; \forall\, z \in \mathbb{R}^p_+\}.
$
\end{definition}

Clearly, every completely positive matrix is positive semidefinite, and every positive semidefinite matrix is copositive, thus $\mathcal{CP}^p \subseteq \mathbb{S}_+^p \subseteq \mathcal{COP}^p$.  

This notion extends naturally to general cones \citep{bai2016conic,nishijima2024approximation}. 

\begin{definition} \label{def:k-copositivity}
Let $\mathbb{K} \subset \mathbb{V}$ be a closed convex cone in the inner product space $\mathbb{V}$. 
The \emph{$\mathbb{K}$-completely-positive cone} is defined as
\begin{equation*}
    \mathcal{CP}_{\mathbb{K}}
    := \operatorname{conv}\{ z \otimes z : z \in \mathbb{K} \}.
\end{equation*}
Its dual,
$
    \mathcal{COP}_{\mathbb{K}}
    = \{A \in \mathrm{Sym}^2(\mathbb{V}) : \langle A, z \otimes z \rangle \geq 0,\; \forall\, z \in \mathbb{K}\},
$
is called the \emph{$\mathbb{K}$-copositive cone}.
\end{definition}

Completely positive cones arise naturally in the convex reformulations of quadratic optimization problems. 
They capture the convex hull of rank-one outer products over a cone, and thus provide a useful tool for lifted representation of certain nonconvex quadratic structures.
We are now ready to introduce the lifted formulation.

\begin{theorem}[Exact CP reformulation] \label{thm:informal-main-result}
Let $d \!:=\! d_N + d_0$ and $r \!:=\! \min_{0 \leq k \leq N} d_k$. Define the selection matrices
\begin{equation}
    \label{eq:pu-pv}
    P_u = \begin{bmatrix} I_{d_N} \\ 0 \end{bmatrix} \in \mathbb{R}^{d \times d_N}, \quad
    P_v = \begin{bmatrix} 0 \\ I_{d_0} \end{bmatrix} \in \mathbb{R}^{d \times d_0}.
\end{equation}
Define $\widehat{\mathbb{K}} := \mathbb{R}_+ \times \operatorname{vec}(\mathbb{S}_+^d)\times \operatorname{vec}(\mathbb{S}_+^d)\times \operatorname{vec}(\mathbb{S}_+^d)$, where we used $\operatorname{vec}(\mathbb{S}_+^d)$ to denote the image of $\mathbb{S}_+^d$ under the $\operatorname{vec}(\cdot)$ operator. 
Then, \eqref{eqn:deep-linear-net} is equivalent to the following convex program:
\begin{equation} \label{eqn:exact-cp-form} 
\begin{aligned}
    \min_{z_i, Z_{ij}} \quad
    & \frac{1}{2n}\, \Big( \langle P_v X X^\top P_v^\top \otimes P_u P_u^\top, Z_{11} \rangle  - 2 \langle \operatorname{vec}(P_u Y X^\top P_v^\top), z_1 \rangle + \| Y \|_F^2 \Big) \\
    \mathrm{s.t.} \quad~\,
    & \operatorname{tr}(\operatorname{mat}(z_2)) = r, \\
    & \langle \operatorname{vec}(I_d) \operatorname{vec}(I_d)^\top, Z_{22} \rangle = r^2, \\
    & z_2 + z_3 = \operatorname{vec}(I_d), \\
    & \operatorname{diag}(Z_{22}+Z_{23}+Z_{32}+Z_{33}) = \operatorname{vec}(I_d),  \\
    & \operatorname{tr}(\operatorname{mat}(z_{1})) - \operatorname{tr}(Z_{12}) = 0, \\
    & \begin{bmatrix}
        1 & z_1^\top  & z_2^\top & z_3^\top  \\
        z_1 & Z_{11} & Z_{12} & Z_{13} \\
        z_2 & Z_{21} & Z_{22} & Z_{23} \\
        z_3 & Z_{31} & Z_{32} & Z_{33} 
    \end{bmatrix} \in \mathcal{CP}_{\widehat{\mathbb{K}}}
\end{aligned}
\end{equation}
in the sense that it has the same optimal value as \eqref{eqn:deep-linear-net}. 
\end{theorem}

\begin{remark}[Exactness, tractability, and recovery]
The equivalence established by the lifted CP formulations is at the level of optimal values. 
Optimization over generalized completely positive cones is computationally difficult in general; however, the particular CP instances obtained here are indirectly solvable through the equivalent reduced-rank regression problem.
Thus, an optimal reduced-rank solution can be lifted to an optimal feasible point of the CP formulation. 
In the reverse direction, if the optimal reduced-rank predictor for the problem studied here is unique, then an arbitrary optimal CP solution recovers it through \(\bar M := P_u^\top \operatorname{mat}(z_1)P_v\). 
An optimal network can then be obtained by factorization.
Without uniqueness, \(\bar M\) lies in the convex hull of optimal reduced-rank predictors and can have rank larger than \(r\). 
Recovering network weights from an arbitrary optimal CP solution in this case may require a generalized completely positive decomposition, which is computationally intractable in general.
\end{remark}

\begin{remark}[Size of the lifted formulation]
Let $d=d_0+d_N$. In the full-vectorized representation, the base cone $\smash{\widehat{\mathbb{K}}}$ has ambient coordinate dimension $p=1+3d^2$, so the lifted completely positive variable is a symmetric matrix of order $p$. The affine side constraints in \eqref{eqn:exact-cp-form} consist of two scalar equalities, two vector-valued equalities of length $d^2$, and one lifted complementarity equality, for a total of $2d^2+3$ scalar equalities, excluding the cone constraint. These counts are independent of the network depth and the number of samples; the samples affect only the objective coefficients. We use full-vectorization for notational simplicity; redundant coordinates could be removed by working directly in the appropriate symmetric tensor spaces, without changing the formulation or the dimension-independence claims.
\end{remark}

The proof of \Cref{thm:informal-main-result} proceeds in three steps. 
First, we reduce the deep linear network training problem to an equivalent shallow factorization, and we express the bilinear factorization through a rank-constrained positive semidefinite matrix. 
Second, we reformulate the rank constraint as a quadratic complementarity condition. 
Finally, we lift this formulation to a completely positive program over the cone $\mathcal{CP}_{\widehat{\mathbb{K}}}$.

\subsection{Rank-Constrained Semidefinite Reformulation}

Recall the model problem in \eqref{eqn:deep-linear-net} for deep linear neural networks. 
We first pass to the equivalent shallow factorization
\begin{equation} 
    \min_{\substack{U \in \mathbb{R}^{d_N \times r} \\ V \in \mathbb{R}^{d_0 \times r}}} ~ 
    \frac{1}{2n}\, \| U V^\top X - Y \|_F^2,
    \label{eqn:shallow}
\end{equation}
where $r=\min_{0 \leq k \leq N} d_k$. 
This equivalence is standard for deep linear networks \citep{saxe2013exact,kawaguchi2016deep}, and it follows from the fact that the composition $W_N\cdots W_1$ defines a linear map of rank at most $r$, while any matrix of rank at most $r$ admits a factorization of the form $UV^\top$.

This bilinear structure can be lifted to a rank-constrained positive semidefinite matrix.
Define
\begin{equation*}
    W = \begin{bmatrix} U \\ V \end{bmatrix} \begin{bmatrix} U \\ V \end{bmatrix}^\top
    = \begin{bmatrix} UU^\top & UV^\top \\ VU^\top & VV^\top \end{bmatrix}
    \in \mathbb{S}_+^{d},
\end{equation*}
so that we obtain $P_u^\top W P_v = UV^\top$ by using the block selection matrices $P_u$ and $P_v$ from \eqref{eq:pu-pv}. 
This leads to the following rank-constrained semidefinite programming formulation
\begin{equation}
    \begin{aligned}
    \min_{W \in \mathbb{S}^d}\quad
    & \frac{1}{2n}\,\|P_u^\top W P_v X - Y\|_F^2 \\
    \mathrm{s.t.}\quad
    & \operatorname{rank}(W)\le r, ~~ W \succeq 0.
    \end{aligned}   
    \label{eqn:rank-constrained}
\end{equation}
Problems in \eqref{eqn:shallow} and \eqref{eqn:rank-constrained} are equivalent  based on the following standard result.

\begin{lemma}
\label{lem:block-psd}
Define the symmetric positive semidefinite block matrix
$ 
    W = 
    \left[\begin{smallmatrix}
        A & M \\
        M^\top & B
    \end{smallmatrix}\right]
    \in \mathbb{S}_+^{\,d_N+d_0}
$
with $A \in \mathbb{S}_+^{d_N}$ and $B \in \mathbb{S}_+^{d_0}$. 
Fix $r \in \mathbb{N}$, then the following sets are equivalent up to factorization via
$M = U V^\top\!\!:$
\begin{equation*}
\Big\{\, (U,V): U \in \mathbb{R}^{d_N \times r},\; V \in \mathbb{R}^{d_0 \times r} \,\Big\} \;\longleftrightarrow\;
\Big\{\, W \in \mathbb{S}_+^{\,d_N+d_0} : \operatorname{rank}(W) \leq r \,\Big\}.
\end{equation*}
\end{lemma}

\begin{proof}
($\Rightarrow$) Given $(U,V)$, set $Z := \left[\begin{smallmatrix} U \\ V \end{smallmatrix}\right] \in \mathbb{R}^{(d_N+d_0)\times r}$.
Then $W = ZZ^\top \succeq 0$ with $\operatorname{rank}(W) \leq r$, and its blocks satisfy
$A = UU^\top$, $B = VV^\top$, $M = UV^\top$.

($\Leftarrow$) Given $W \succeq 0$ with $\operatorname{rank}(W) \leq r$, take any factor $W = ZZ^\top$
with $Z \in \mathbb{R}^{(d_N+d_0)\times r}$ and partition $Z = \left[\begin{smallmatrix} U \\ V \end{smallmatrix}\right]$.
Then $A = UU^\top$, $B = VV^\top$, and $M = UV^\top$.
\end{proof}

\begin{remark}[Expressivity vs.\ optimization bias]
Although both \eqref{eqn:deep-linear-net} and \eqref{eqn:shallow} represent the same family of linear mappings, the deep parameterization can induce distinct optimization dynamics and implicit regularization~\citep{arora2019implicit,feng2022rank}.  
Our focus here is on expressivity and on the equivalence of training problems at the level of global optima.
\end{remark}

\subsection{Complementarity Reformulation}

We next reformulate the rank constraint via a complementarity condition. 

\begin{lemma}\label{lem:complementarity}
A matrix $W \in \mathbb{S}_+^d$ satisfies
$\operatorname{rank}(W) \le r$ for some $r\in\{0,\dots,d\}$ if and only if there exists
$W' \in \mathbb{S}_+^d$ such that
\begin{equation}
    \label{eqn:complementarity-conditions}
    \operatorname{tr}(W') = r, \quad 
    0 \preceq W' \preceq I, \quad
    \langle W, I - W' \rangle \leq 0.
\end{equation}
\end{lemma}

\begin{proof}
($\Rightarrow$) If $\operatorname{rank}(W) \le r$, choose an $r$-dimensional subspace containing $\operatorname{Ran}(W)$ and let $W'$ be the orthogonal projector onto this subspace.
Then $W'$ satisfies the conditions.

($\Leftarrow$) Suppose the conditions in \eqref{eqn:complementarity-conditions} hold. 
Since $W \succeq 0$ and $I-W' \succeq 0$, we have $\langle W, I-W' \rangle \geq 0$. 
Together with $\langle W, I-W' \rangle \leq 0$, this gives $\langle W, I-W' \rangle = 0$, and hence $W (I-W') = 0$.
This shows that $\operatorname{Ran}(W) \subseteq \ker(I-W')$, i.e., the range of $W$ is contained in the eigenspace of $W'$ corresponding to eigenvalue $1$. Since $0 \preceq W' \preceq I$, all eigenvalues of $W'$ are in $[0,1]$, hence the multiplicity of the eigenvalue $1$ is at most $\operatorname{tr}(W') = r$. 
Thus, $\operatorname{rank}(W) = \dim(\operatorname{Ran}(W)) \le \dim(\ker(I-W')) \le r$. 
\end{proof}

Using \Cref{lem:complementarity}, the rank-constrained formulation
\eqref{eqn:rank-constrained} can be written equivalently as
\begin{equation} 
    \begin{aligned}
    \min_{S, W, W' \in \mathbb{S}^d} \quad
    & \frac{1}{2n}\,\|P_u^\top W P_v X - Y\|_F^2 \\
    \mathrm{s.t.}\quad~~~
    & \operatorname{tr}(W') = r, \\
    & W' + S = I_d, \\
    & \langle W, S \rangle \leq 0, \\
    & W \succeq 0,\quad W' \succeq 0,\quad S \succeq 0.
    \end{aligned}
    \label{eqn:complementarity-formulation}
\end{equation}
In particular, \eqref{eqn:rank-constrained} and \eqref{eqn:complementarity-formulation} have the same feasible set in the variable $W$, and therefore the same optimal value. 

\subsection{Completely Positive Lifting}

We now lift the complementarity formulation to a completely positive program. 
We first express the problem in vectorized form in order to apply \Cref{lem:bai-lift}. 
We then rewrite the resulting formulation in a matrix-based representation using Kronecker products, which provides a more natural interpretation of the lifted variables. 

\subsubsection{Vectorized Formulation}

Problem \eqref{eqn:complementarity-formulation}  admits an exact reformulation as a completely positive program via a quadratic lifting.
The following lemma is a simplified specialization of Theorem~4 in \citep{bai2016conic}, stated in the form needed for our application.

\begin{lemma}[Completely positive lifting] \label{lem:bai-lift} 
Let $\mathbb{K} \subset \mathbb{R}^p$ be a closed convex cone and define the augmented cone $\widehat{\mathbb{K}} := \mathbb{R}_+ \times \mathbb{K}$. 
Let $h,\bar h \in \mathbb{R}$, $q,\bar q \in \mathbb{R}^p$, and $Q,\bar Q \in \mathbb{R}^{p\times p}$ be given, with $Q \in \mathcal{COP}_{\mathbb{K}}$.
Let $A \in \mathbb{R}^{m\times p}$ and $b \in \mathbb{R}^m$, and write $a_i^\top$ for the $i$-th row of $A$ and $b_i$ for the corresponding entry of $b$. 
Consider the quadratically constrained quadratic program
\begin{equation}
\begin{aligned}
    \min_{z \in \mathbb{R}^p} \quad 
    & z^\top Q z + q^\top z + h \\
    \quad \operatorname{s.t.} \quad
    & a_i^\top z = b_i, \quad \text{for $i=1\dots,m$} \\
    \quad & z^\top \bar{Q} z + \bar{q}^\top z + \bar{h} \leq 0, \\
    \quad & z \in \mathbb{K},
\end{aligned}
\label{eqn:general-qp}
\end{equation}
and suppose that its feasible set contains no point at which the quadratic inequality constraint is satisfied strictly.
Then, \eqref{eqn:general-qp} is equivalent to the following completely positive program:
\begin{equation}
\begin{aligned}
    \min_{z \in \mathbb{R}^p, Z \in \mathbb{S}^{p}} \quad &
    \langle Q, Z \rangle + q^\top z + h \\
    \quad \operatorname{s.t.} \quad \quad
    & a_i^\top z = b_i, \quad \text{for $i=1,\dots,m$} \\
    & \langle a_i a_i^\top, Z \rangle = b_i^2, \quad \text{for $i=1,\dots,m$} \\
    & \langle \bar{Q},  Z \rangle  + \bar{q}^\top z + \bar{h} = 0, \quad \\ 
    & \begin{bmatrix} 1 & z^\top \\ z & Z \end{bmatrix} \in \mathcal{CP}_{\widehat{\mathbb{K}}},
\end{aligned}
\label{eqn:general-qp-cp-equiv}
\end{equation}
in the sense that problems \eqref{eqn:general-qp} and \eqref{eqn:general-qp-cp-equiv} have the same optimal value. Moreover, if $(z,Z)$ is optimal for \eqref{eqn:general-qp-cp-equiv}, then $z$ belongs to the convex hull of the set of optimal solutions of \eqref{eqn:general-qp}.  
\end{lemma}

\begin{proof}
The quadratic problem \eqref{eqn:general-qp} is the nSp0-QCQP class of \citet[Definition~1, Equation~(4)]{bai2016conic}, specialized to the case with no additional convex quadratic constraints (\(J=0\) in their notation). 
For this class, the quadratic term $\omega(z) := z^\top \bar Q z+\bar q^\top z+\bar h$ is required by their Definition~1 to be nonnegative on $\mathbb K\cap\{z:Az=b\}$. 
Equivalently, no feasible point of \eqref{eqn:general-qp} satisfies $\omega(z)\le 0$ strictly, which is precisely the no strict-feasibility assumption in \Cref{lem:bai-lift}. 
Under this specialization, the lifted completely positive program~\eqref{eqn:general-qp-cp-equiv} corresponds to the lifted formulation in their Equation~(12). 

Theorem~4 of \citet{bai2016conic} assumes that \(Q\) is copositive on
\(\mathbb L := \{z\in\mathbb K: Az=0,\ z^\top \bar Q z=0\}\), which is a subset of the relevant recession directions.
The stronger condition in \Cref{lem:bai-lift}, namely \(Q\in\mathcal{COP}_{\mathbb K}\), implies this requirement since \(\mathbb L\subseteq\mathbb K\).
Consequently, their theorem applies and gives the optimal-value equivalence as well as the convex-hull statement in \Cref{lem:bai-lift}.
\end{proof}

We now apply \Cref{lem:bai-lift} to the complementarity reformulation in \eqref{eqn:complementarity-formulation}. 
To this end, we identify the finite-dimensional matrix space $\mathbb{S}^d \times \mathbb{S}^d \times \mathbb{S}^d$ with a Euclidean space via vectorization, so that the complementarity formulation fits the setting of \Cref{lem:bai-lift}. 
Define 
\begin{equation*}
    z :=
    \begin{bmatrix}
        \operatorname{vec}(W) \\ \operatorname{vec}(W') \\ \operatorname{vec}(S)
    \end{bmatrix}
    \in \mathbb{R}^{3d^2},
    \qquad
    \mathbb{K} := \operatorname{vec}(\mathbb{S}_+^d)\times \operatorname{vec}(\mathbb{S}_+^d)\times \operatorname{vec}(\mathbb{S}_+^d)
    \subset \mathbb{R}^{3d^2},
\end{equation*}
where $\operatorname{vec}(\mathbb{S}_+^d)$ denotes the image of $\mathbb{S}_+^d$ under the $\operatorname{vec}(\cdot)$ operator. 
We introduce new block-selection matrices 
\begin{equation*}
\begin{aligned}
    R_1 & := \begin{bmatrix} I_{d^2} & 0_{d^2} & 0_{d^2} \end{bmatrix} \in \mathbb{R}^{d^2 \times 3d^2}
    \\ 
    R_2 & := \begin{bmatrix} 0_{d^2} & I_{d^2} & 0_{d^2} \end{bmatrix} \in \mathbb{R}^{d^2 \times 3d^2}
    \\ 
    R_3 & := \begin{bmatrix} 0_{d^2} & 0_{d^2} & I_{d^2} \end{bmatrix} \in \mathbb{R}^{d^2 \times 3d^2} 
\end{aligned}
\end{equation*}
so that $R_1 z = \operatorname{vec}(W)$, $R_2 z = \operatorname{vec}(W')$ and $R_3 z = \operatorname{vec}(S)$, using full-vectorization for notational simplicity.

We are now ready to express problem \eqref{eqn:complementarity-formulation} in terms of $z$. 

\begin{proposition} \label{prop:complementarity-qcqp}
Problem \eqref{eqn:complementarity-formulation} is equivalent to problem \eqref{eqn:general-qp} with 
\begin{equation*}
\begin{aligned}
    & Q = \frac{1}{2n} R_1^\top \bigl((P_v X X^\top P_v^\top)\otimes (P_u P_u^\top)\bigr) R_1, & \quad & \bar{Q} = -\frac{1}{2}\bigl(R_1^\top R_2 + R_2^\top R_1\bigr), & \quad & A = \begin{bmatrix} \operatorname{vec}(I_d)^\top R_2 \\ R_2 + R_3
    \end{bmatrix}
    \\
    & q = -\frac{1}{n} R_1^\top \operatorname{vec}(P_u Y X^\top P_v^\top), & & \bar{q} = R_1^\top \operatorname{vec}(I_d), & & b = \begin{bmatrix} r \\ \operatorname{vec}(I_d) \end{bmatrix}. \\
    & h = \frac{1}{2n} \|Y\|_F^2, & & \bar{h} = 0,\\
\end{aligned}
\end{equation*}
Moreover, $Q \in \mathcal{COP}_{\mathbb{K}}$ and the quadratic constraint satisfies the non-strict feasibility condition in \Cref{lem:bai-lift}. 
Therefore, the assumptions are satisfied, and the formulation is equivalent to the lifted problem \eqref{eqn:general-qp-cp-equiv}.
\end{proposition}

\begin{proof}
Using the identity $\operatorname{vec}(A W B) = (B^\top \otimes A)\operatorname{vec}(W)$, we have
$ \operatorname{vec}(P_u^\top W P_v X) = (X^\top P_v^\top \otimes P_u^\top) R_1 z. $

Expanding the objective, we obtain
\begin{equation*}
    \frac{1}{2n} \|P_u^\top W P_v X - Y\|_F^2 
    = \frac{1}{2n} \bigg( \|P_u^\top W P_v X \|_F^2 - 2 \langle P_u^\top W P_v X, Y \rangle + \|Y\|_F^2 \bigg).
\end{equation*}
The quadratic term satisfies
\begin{equation*}
\begin{aligned}
    \|P_u^\top W P_v X \|_F^2
    & = \operatorname{vec}(P_u^\top W P_v X)^\top \operatorname{vec}(P_u^\top W P_v X) \\
    & = z^\top R_1^\top (X^\top P_v^\top \otimes P_u^\top)^\top (X^\top P_v^\top \otimes P_u^\top) R_1 z \\
    & = z^\top R_1^\top (P_v X \otimes P_u) (X^\top P_v^\top \otimes P_u^\top) R_1 z \\
    & = z^\top R_1^\top (P_v X X^\top P_v^\top \otimes P_u P_u^\top) R_1 z.
\end{aligned}
\end{equation*}
For the linear term, we have
\begin{equation*}
\begin{aligned}
    \langle P_u^\top W P_v X, Y \rangle
    = \langle W , P_u Y X^\top P_v^\top \rangle
    = z^\top R_1^\top \operatorname{vec}(P_u Y X^\top P_v^\top).
\end{aligned}
\end{equation*}
Hence we obtain the terms $Q$, $q$ and $h$ in \Cref{prop:complementarity-qcqp}. 
Moreover, since $(P_v X X^\top P_v^\top)\succeq 0$ and $(P_u P_u^\top)\succeq 0$, their Kronecker product is also positive semidefinite, hence $Q \succeq 0$ and therefore $Q \in \mathcal{COP}_{\mathbb{K}}$.

Looking at the affine constraints, we have $\operatorname{tr}(W') = \langle I_d , W' \rangle = \operatorname{vec}(I_d)^\top R_2 z$ and $S+W' = I_d$. Hence, it can be written in the form $Az = b$, with $A$ and $b$ as defined in \Cref{prop:complementarity-qcqp}.

Finally, we reformulate the quadratic constraint by 
\begin{equation*}
\begin{aligned}
    \langle W, I_d - W' \rangle 
    & = \langle W, I_d \rangle - \frac{1}{2} \big( \langle W, W' \rangle + \langle W', W \rangle \big) \\
    & = \operatorname{vec}(I_d)^\top R_1 z - \frac{1}{2} \big( z^\top R_1^\top R_2 z + z^\top R_2^\top R_1 z \big)
\end{aligned}
\end{equation*}
leading to the terms $\bar{Q}$, $\bar{q}$, $\bar{h}$ in \Cref{prop:complementarity-qcqp}. 
Since $W,S \succeq 0$, we have $\langle W, S \rangle \geq 0$, and the quadratic constraint is satisfied at equality for all feasible points of \eqref{eqn:complementarity-formulation}, hence the non-strict feasibility condition holds. 
This establishes the equivalence.
\end{proof}

The problem formulation in \Cref{thm:informal-main-result} is an algebraic reformulation of \Cref{prop:complementarity-qcqp}, expressed in a more structured and interpretable form. 

\subsubsection{Kronecker Representation}

The same construction admits a tensor interpretation. 
In this subsection, we use this identification to write the lifted blocks in Kronecker form. 

Let $\bar{\mathbb{K}} = \mathbb{R}_+ \oplus \mathbb{S}_+^d \oplus \mathbb{S}_+^d \oplus \mathbb{S}_+^d$.
Consider $\mathcal{W} \in \mathcal{CP}_{\bar{\mathbb{K}}}$ with the blocks corresponding to the components $(\omega_0, \omega_1, \omega_2, \omega_3) \in \bar{\mathbb{K}}$.
The vectorization map identifies $\bar{\mathbb{K}}$ with the cone $\widehat{\mathbb{K}}$ used in \Cref{thm:informal-main-result}. 
Under this identification, a vectorized lifted block $Z_{ij}$ is equivalently a matrix-indexed block $\mathcal{W}_{ij}$. 
On rank-one generating atoms $\omega=(t,A,B,C)$, with $t\ge0$ and $A,B,C\succeq0$, the corresponding lifted blocks are $\mathcal{W}_{ij}=\omega_i\otimes\omega_j$, e.g.,
\begin{equation*}
    \mathcal{W}_{01} = tA,\quad
    \mathcal{W}_{02} = tB,\quad
    \mathcal{W}_{03} = tC,\quad
    \mathcal{W}_{11} = A \otimes A,\quad
    \mathcal{W}_{12} = A \otimes B,\quad
    \mathcal{W}_{22} = B \otimes B,\quad
    \mathcal{W}_{23} = B \otimes C,
\end{equation*}
and the remaining blocks are interpreted analogously.
Thus \(\mathcal W_{0j}\) and \(\mathcal W_{j0}\) for \(j\ge1\) are represented as \(d\times d\) matrices, while \(\mathcal W_{ij}\) for \(i,j\ge1\) are represented as \(d^2\times d^2\) matrices. Traces and inner products below are the ordinary matrix trace and Frobenius inner product in these representations.

To express the vectorized quadratic coefficient against standard Kronecker-product lifted blocks, we define the reshuffling operator $\mathcal{R}: \mathbb{R}^{d^2\times d^2}\to\mathbb{R}^{d^2\times d^2}$, for $M\in\mathbb{R}^{d^2\times d^2}$, by
\[
    \bigl( \mathcal{R}(M) \bigr)_{(i,k),(j,\ell)} = M_{(i,j),(k,\ell)}, \quad \text{for all $i,j,k,\ell$.}
\]
$\mathcal{R}$ is a permutation of the entries of its matrix argument, i.e., there exists a permutation matrix $\Pi\in\mathbb{R}^{d^4\times d^4}$ such that $\operatorname{vec}(\mathcal{R}(M))=\Pi\operatorname{vec}(M)$. 
It changes the matricization of a fourth-order tensor from the $(i,j),(k,\ell)$ grouping induced by $\operatorname{vec}(W)\operatorname{vec}(W)^\top$ to the $(i,k),(j,\ell)$ grouping induced by the standard Kronecker product $W\otimes W$. 
In particular,
\[
    \left\langle \mathcal{R}(B\otimes A), W\otimes W \right\rangle = \langle B\otimes A, \operatorname{vec}(W) \operatorname{vec}(W)^\top \rangle.
\]

\noindent We are now ready to present our Kronecker formulation, which amounts to a coordinate reshaping of the formulation in \Cref{thm:informal-main-result}. 

\begin{corollary}[Exact CP reformulation] \label{thm:main-result-kronecker}
\eqref{eqn:deep-linear-net} is equivalent to the following convex program:
\begin{equation}
\label{eqn:deep-nn-cp-equiv}
\begin{aligned} 
    \min_{\mathcal{W}_{ij}} \quad
    & \frac{1}{2n}\, \Big( \langle \mathcal{R}(P_v X X^\top P_v^\top \otimes P_u P_u^\top), \mathcal{W}_{11} \rangle - 2 \langle P_u Y X^\top P_v^\top, \mathcal{W}_{01} \rangle + \| Y \|_F^2 \Big) \\
    \quad \mathrm{s.t.} ~~~
    & \operatorname{tr}(\mathcal{W}_{02}) = r, \\
    & \operatorname{tr}(\mathcal{W}_{22}) = r^2, \\
    & \langle E_{ij},\mathcal{W}_{02}+\mathcal{W}_{03}\rangle = \delta_{ij} , ~~ \forall i,j, \\
    & \langle E_{ij} \otimes E_{ij} , \mathcal{W}_{22} + \mathcal{W}_{23} + \mathcal{W}_{32} + \mathcal{W}_{33} \rangle = \delta_{ij} , ~~ \forall i,j, \\
    & \operatorname{tr}(\mathcal{W}_{01}) - \langle \Delta, \mathcal{W}_{12} \rangle = 0, \\
    & \begin{pmatrix}
        1 & \mathcal{W}_{01} & \mathcal{W}_{02} & \mathcal{W}_{03} \\
        \mathcal{W}_{10} & \mathcal{W}_{11} & \mathcal{W}_{12} & \mathcal{W}_{13} \\
        \mathcal{W}_{20} & \mathcal{W}_{21} & \mathcal{W}_{22} & \mathcal{W}_{23} \\
        \mathcal{W}_{30} & \mathcal{W}_{31} & \mathcal{W}_{32} & \mathcal{W}_{33} 
    \end{pmatrix} \in \mathcal{CP}_{\bar{\mathbb{K}}},
\end{aligned}
\end{equation}
in the sense that \eqref{eqn:deep-nn-cp-equiv} has the same optimal value as
\eqref{eqn:deep-linear-net}. 
Here $\delta_{ij}$ is the indicator that takes $1$ if $i=j$ and $0$ otherwise, $\{E_{ij}\}_{i,j}$ denotes the standard matrix basis, $E_{ij} \otimes E_{ij}$ denotes the corresponding Kronecker basis elements, and $\Delta = \sum_{i=1}^d \sum_{j=1}^d E_{ij} \otimes E_{ij}$ is the diagonal contraction operator satisfying $\langle \Delta, A \otimes B \rangle = \langle A, B \rangle$ for all $A,B \in \mathbb{R}^{d \times d}$. 
The formulation is linear in the lifted variable $\mathcal{W}$, with all nonconvexity encoded in the completely positive constraint.
\end{corollary}

\begin{proof}
The vectorization map induces a linear isomorphism between the lifted cones
$\smash{\mathcal{CP}_{\bar{\mathbb{K}}}}$ and $\smash{\mathcal{CP}_{\widehat{\mathbb{K}}}}$.
Indeed, it maps each generating atom
$\omega=(t,A,B,C)\in\bar{\mathbb{K}}$ to
$(t,\operatorname{vec}(A),\operatorname{vec}(B),\operatorname{vec}(C))\in\widehat{\mathbb{K}}$,
and the inverse map is obtained by reshaping each vectorized block back into a matrix.
Since every expression in the two lifted formulations \eqref{eqn:exact-cp-form} and \eqref{eqn:deep-nn-cp-equiv} is affine in the lifted variables and every element of the completely positive cone is a conic combination of generating atoms, it suffices to verify the objective and constraint identities on the generating atoms.
The identities then extend by linearity to arbitrary feasible lifted variables.

First consider the objective.  
For the quadratic term, 
\begin{equation*}
\begin{aligned}
    \langle P_v XX^\top P_v^\top \otimes P_u P_u^\top, Z_{11} \rangle 
    & = \langle P_v XX^\top P_v^\top \otimes P_u P_u^\top, \operatorname{vec}(A) \operatorname{vec}(A)^\top \rangle \\
    & = \langle \mathcal{R}(P_v XX^\top P_v^\top \otimes P_u P_u^\top), A \otimes A \rangle \\
    & = \langle \mathcal{R}(P_v XX^\top P_v^\top \otimes P_u P_u^\top), \mathcal{W}_{11} \rangle.
\end{aligned}
\end{equation*}
For the linear term, the first-order block relation is
$z_1= tA =\operatorname{vec}(\mathcal{W}_{01})$, and hence
\begin{equation*}
    \langle \operatorname{vec}(P_uYX^\top P_v^\top),z_1\rangle
    = \langle P_uYX^\top P_v^\top,\mathcal{W}_{01}\rangle .
\end{equation*}
This proves the equivalence of the objective terms.

The affine constraints are mapped similarly. 
Since $z_2=tB=\operatorname{vec}(\mathcal{W}_{02})$, the constraint in the vectorized formulation $\operatorname{tr}(\operatorname{mat}(z_2))=r$ corresponds exactly to $\operatorname{tr}(\mathcal{W}_{02})=r$ in the Kronecker formulation. 
Likewise, the equality constraint $z_2+z_3=\operatorname{vec}(I_d)$, which can be written componentwise as $\langle \operatorname{vec}(E_{ij}),z_2+z_3\rangle=\delta_{ij}$ for $i,j=1,\dots,d$, becomes $\langle E_{ij},\mathcal{W}_{02}+\mathcal{W}_{03} \rangle=\delta_{ij}$ in the Kronecker notation.

For the lifted trace constraint, on a generating atom
\(\omega=(t,A,B,C)\), we have
\begin{equation*}
\begin{aligned}
    \langle \operatorname{vec}(I_d)\operatorname{vec}(I_d)^\top,Z_{22}\rangle 
    & = \langle \operatorname{vec}(I_d)\operatorname{vec}(I_d)^\top, \operatorname{vec}(B)\operatorname{vec}(B)^\top \rangle  \\
    & = \langle I_d, B \rangle^2 \\
    & = \langle I_d \otimes I_d , B \otimes B \rangle \\
    & = \langle I_d \otimes I_d , \mathcal{W}_{22} \rangle
    = \operatorname{tr}(\mathcal{W}_{22}).
\end{aligned}
\end{equation*}
Thus the constraint in the vectorized formulation 
\(\langle \operatorname{vec}(I_d)\operatorname{vec}(I_d)^\top,Z_{22}\rangle=r^2\)
corresponds to the constraint
\(\operatorname{tr}(\mathcal{W}_{22})=r^2\) in the Kronecker formulation.

For the lifted square of \(W'+S=I_d\), the vectorized formulation uses the $\operatorname{diag}(Z_{22}+Z_{23}+Z_{32}+Z_{33})=\operatorname{vec}(I_d)$ constraint. 
Componentwise, this is equivalent to 
$ \langle\operatorname{vec}(E_{ij})\operatorname{vec}(E_{ij})^\top, Z_{22}+Z_{23}+Z_{32}+Z_{33}\rangle = \delta_{ij},$ for all $i,j = 1,\dots,d$. 
On a generating atom \(\omega=(t,A,B,C)\), we have
\begin{equation*}
\begin{aligned}
    \left\langle \operatorname{vec}(E_{ij})\operatorname{vec}(E_{ij})^\top, Z_{22}+Z_{23}+Z_{32}+Z_{33} \right\rangle 
    & = \left\langle E_{ij}\otimes E_{ij}, B\otimes B+B\otimes C+C\otimes B+C\otimes C \right\rangle \\
    & = \left\langle E_{ij}\otimes E_{ij}, \mathcal W_{22}+\mathcal W_{23}+\mathcal W_{32}+\mathcal W_{33} \right\rangle .
\end{aligned}
\end{equation*}
Thus we obtain precisely $\left\langle E_{ij}\otimes E_{ij}, \mathcal W_{22}+\mathcal W_{23}+\mathcal W_{32}+\mathcal W_{33} \right\rangle = \delta_{ij}$ in the Kronecker formulation. 

It remains to identify the lifted complementarity constraint. On a generating atom $\omega=(t,A,B,C)$, the vectorized expression satisfies
\begin{equation*}
\begin{aligned}
    \operatorname{tr}(\operatorname{mat}(z_1))-\operatorname{tr}(Z_{12})
    & = \operatorname{tr}(tA) - \operatorname{tr}\!\left(\operatorname{vec}(A)\operatorname{vec}(B)^\top\right) \\
    & = \operatorname{tr}(tA)-\langle A,B\rangle \\
    & = \operatorname{tr}(tA)-\langle\Delta,A\otimes B\rangle \\
    & = \operatorname{tr}(\mathcal W_{01})-\langle\Delta,\mathcal W_{12}\rangle .
\end{aligned}
\end{equation*}
Thus the lifted complementarity equality in the vectorized formulation is exactly the corresponding equality in the Kronecker formulation.

Since the lifted objectives and constraints are linear in the lifted variables, verifying the identities on generating atoms proves the identities for all conic combinations. 
Hence the lifted feasible sets correspond bijectively under the vectorization isomorphism. 
The objective values are preserved by the same identities, so the two formulations are equivalent.
\end{proof}

\section{Conclusion}

We have shown that the squared-loss training problem for deep linear neural networks admits an exact convex reformulation after lifting to a generalized completely positive cone.
The resulting convex program has the same global optimal value as the original nonconvex training problem, and the size of the lifted representation depends on the input and output dimensions rather than the network depth or number of samples, with the bottleneck width appearing only in scalar constraints. 

For the studied unconstrained squared-loss problem studied here, reduced-rank regression remains the direct computational route to an optimal predictor. 
The main value of the proposed reformulation is therefore structural, as it identifies the precise conic object capturing the nonconvexity induced by linear factorization. 
A natural direction for future work is to investigate whether analogous liftings are useful for constrained or regularized linear-network training problems, where the reduced-rank-regression solution no longer directly applies. 
In such settings, tractable relaxations of the resulting CP formulations may become computationally meaningful.

\section*{Acknowledgments}

This work was supported by the Wallenberg AI, Autonomous Systems and Software Program (WASP) funded by the Knut and Alice Wallenberg Foundation.

\small
\bibliographystyle{plainnat}
\bibliography{bibliography}

\end{document}